\documentclass[conference]{IEEEtran}
\IEEEoverridecommandlockouts
\usepackage{amsmath,amssymb,amsfonts}

\usepackage[linesnumbered,ruled]{algorithm2e}

\usepackage{enumitem}
\usepackage{multirow}

\usepackage{mathtools}				
\DeclareMathOperator{\diag}{diag}	
\DeclareMathOperator{\relu}{ReLU}	
\DeclareMathOperator{\tr}{tr}	    
\DeclareMathOperator{\vecc}{vec}	    
\DeclareMathOperator{\idx}{idx}	    
\DeclareMathOperator*{\argmax}{arg\,max}        

\usepackage{pgf,tikz,pgfplots}
\pgfplotsset{compat=1.15}
\usepackage{mathrsfs}
\usetikzlibrary{arrows,quotes,angles}

\usepackage{amsthm}
\theoremstyle{plain}
\newtheorem{theorem}{Theorem}
\newtheorem{proposition}{Proposition}
\newtheorem{corollary}{Corollary}
\newtheorem{lemma}{Lemma}
\theoremstyle{definition}

\theoremstyle{remark}

\usepackage{xcolor}
\usepackage{subcaption}
\usepackage{caption}
\usepackage{float}

\def\BibTeX{{\rm B\kern-.05em{\sc i\kern-.025em b}\kern-.08em
    T\kern-.1667em\lower.7ex\hbox{E}\kern-.125emX}}
\begin{document}

\title{A Sequential Framework Towards an Exact SDP Verification of Neural Networks\\
\thanks{This work was supported by grants from AFOSR, ONR and NSF.}
}

\author{\IEEEauthorblockN{1\textsuperscript{st} Ziye Ma}
\IEEEauthorblockA{\textit{Electrical Engineering and Computer Sciences} \\
\textit{UC Berkeley}\\
Berkeley, CA, USA \\
ziyema@berkeley.edu}
\and
\IEEEauthorblockN{2\textsuperscript{nd} Somayeh Sojoudi}
\IEEEauthorblockA{\textit{Electrical Engineering and Computer Sciences} \\
\textit{Mechanical Engineering} \\
\textit{UC Berkeley}\\
Berkeley, CA, USA \\
sojoudi@berkeley.edu}}


\maketitle

\begin{abstract}
Although neural networks have been applied to several systems in recent years, they still cannot be used in safety-critical systems due to the lack of efficient techniques to certify their robustness.  A number of techniques based on convex optimization have been proposed in the literature to study the robustness of neural networks, and the semidefinite programming (SDP) approach has emerged as a leading contender for the robust certification of neural networks. The major challenge to the SDP approach is that it is prone to a large relaxation gap. In this work, we address this issue by developing a sequential framework to shrink this gap to zero by adding non-convex cuts to the optimization problem via disjunctive programming. We analyze the performance of this sequential SDP method both theoretically and empirically, and show that it bridges the gap as the number of cuts increases.
\end{abstract}

\begin{IEEEkeywords}
neural networks, robustness, safety, convex optimization, semidefinite programming, disjunctive programming
\end{IEEEkeywords}

\section{Introduction}

The nonlinearity of activation functions in neural networks is the key enabler of making neural networks act as universal function approximators, offering great expressive power \cite{sonoda2017neural}. However, this also poses major challenges for the verification against adversarial attacks, due to the non-convexity that activation functions induce. To circumvent the problem, several techniques based on convex relaxations of ReLU constraints have been proposed. In particular, Linear Programming (LP) relaxation \cite{wong2017provable}, Semidefinite Programming (SDP) relaxation \cite{raghunathan2018semidefinite}, and mixed-integer linear programming relaxation \cite{tjeng2017evaluating} \cite{anderson2020tightened,anderson2021partition} have been proposed. LP-based techniques relax each ReLU function individually, thus introducing a relatively large relaxation gap. Although some recent works, such as \cite{singh2019beyond}, have developed k-ReLU relaxations to consider multiple ReLU relaxations jointly, the relaxation gap is still large. On the contrary, SDP-based relaxations naturally couple ReLU relaxations together without any additional effort via a semidefinite constraint. Therefore, as the number of hidden layers of the network under verification grows, the relative reduction in the relaxation gap also grows when compared to LP-based methods. However, even with the power of SDP relaxation, the relaxation gap is still significant under most settings. In this work, we address the issue with the SDP relaxation and shrink the relaxation gap.

The contribution of this paper is four-fold:
\begin{itemize}[noitemsep,topsep=0pt]
	\item A technique to introduce non-convex cuts into the SDP relaxation via secant approximation of non-convex constraints;
	\item An iterative algorithm that certifies a given network with nonnegative reduction in relaxation gap every step until a certificate can be produced;
	\item Theoretical and empirical analyses of the efficacy of several other cut-based techniques for comparison purposes;
	\item Geometrical analysis of the proposed technique using non-convex cuts, to provide insights into the current approach and future improvements.
\end{itemize}

\subsection{Notations}

We denote the set of real-valued $m \times m$ symmetric matrices as $\mathbb{S}^m$, and the notation $X \succeq 0$ means that X is a symmetric positive semidefinite matrix. "$\cdot$" denotes the usual vector dot product. The Hadamard (element-wise) product between $X$ and $Y$ is denoted as $X\odot Y$. Operator $\diag(\cdot)$ coverts its vector argument to a diagonal matrix. $\idx(A,a)$ denotes the position/index of element $a$ in vector or matrix $A$ (e.g., $\idx([1,2,3],3)=3$). $\dagger$ denotes the Penrose-Moore generalized inverse, and $\langle A,B \rangle$ denotes $\tr(A^\top B)$ for square matrices $A,B$.

\section{Problem Statement}
\label{sec:problem_statement}
Consider a $K$-layer ReLU neural network defined by
\begin{equation}
\begin{aligned}
x^{[0]} ={}& x, \ \ x^{[k]} = \relu(W^{[k-1]}x^{[k-1]})
\end{aligned} \label{eqn: network_description}
\end{equation}
for all $k\in\{1,2,\dots,K\}$, where $x\in\mathbb{R}^{n_x}$ is the input to the neural network, $z\triangleq x^{[K]} \in\mathbb{R}^{n_z}$ is the output, and $\hat{x}^{[k]} = W^{[k-1]}x^{[k-1]}+b^{[k-1]}\in\mathbb{R}^{n_k}$ is the preactivation of the $k^{\text{th}}$ layer. The parameters $W^{[k]}\in\mathbb{R}^{n_{k+1}\times n_k}$ and $b^{[k]}\in\mathbb{R}^{n_{k+1}}$ are the weight matrix and bias vector applied to the $k^{\text{th}}$ layer's activation $x^{[k]}\in\mathbb{R}^{n_k}$, respectively. Without loss of generality, assume that the bias terms are accounted for in the activations $x^{[k]}$, thereby setting $b^{[k]}=0$ for all layers $k$. Let the function $f\colon \mathbb{R}^{n_x}\to\mathbb{R}^{n_z}$ denote the mapping $x\mapsto z$ defined by (\ref{eqn: network_description}).

Following the spirits of \cite{wong2017provable,raghunathan2018semidefinite}, define the input uncertainty set $\mathcal{X}\subseteq\mathbb{R}^{n_x}$ using $l_{\infty}$ norms: $\mathcal{X} = \{x\in\mathbb{R}^{n_x} : \|x - \bar{x}\|_\infty \le \epsilon\}$. Similarly, define $\mathcal{S}\subseteq\mathbb{R}^{n_z}$ as the \emph{safe set}. For classification networks, safe sets are (possibly unbounded) usually polyhedral sets defined as the intersection of a finite number of half-spaces:
\begin{equation*}
\mathcal{S} = \{z\in\mathbb{R}^{n_z} : C z \le 0\},
\end{equation*}
where $C\in\mathbb{R}^{n_\mathcal{S}\times n_z}$ is given. Any output $z \in \mathcal{S}$ is said to be $\emph{safe}$. The safe set is assumed to be polyhedral in the rest of this paper.

The goal of certification is to ensure that $f(x) \in \mathcal{S}$ for all $x \in \mathcal{X}$, which is equivalent to checking the satisfaction of the following inequality:
\begin{equation}
\begin{aligned}
\max_{i \in \{1, \dots ,n_{\mathcal{S}}\}} & \{f_i^*(\mathcal{X})\} \leq 0
\end{aligned} \label{eq: robustness_certification_problem}
\end{equation}
where $f_i^*(\mathcal{X}) = \sup\{c_i^\top z : z=f(x), ~ x\in\mathcal{X} \}$. This is a non-convex optimization problem, since the constraint $z=f(x)$ is nonlinear (activation functions are designed to make $f$ nonlinear).

\subsection{SDP Relaxation}

In this section, we present the SDP relaxation used for robustness certification. The details can be found in \cite{raghunathan2018semidefinite}. The main idea is to convert the ReLU constraints to quadratic constraints and then reformulate the non-convex certification problem (\ref{eq: robustness_certification_problem}) as a quadratically-constrained quadratic program (QCQP). Then, the standard SDP relaxation of the resulting QCQP leads to checking whether the following inequality is satisfied:
\begin{equation}
\begin{aligned}
\max_{i \in \{1, \dots ,n_{\mathcal{S}}\}} & \{\hat f_i^*(\mathcal{X})\} \leq 0
\end{aligned} \label{eq: sdp_relaxed_certification_problem}
\end{equation}
where $\hat f_i^*(\mathcal{X}) = \sup\{c_i^\top z : (x,z) \in\mathcal{N}_\text{SDP}, ~ x\in\mathcal{X} \}$. The notation $(x,z) \in\mathcal{N}_\text{SDP}$ means that there exist $\tilde x \in \mathbb{R}^{n_{\tilde x}} \ \text{and} ~  \tilde X \in \mathbb{S}^{n_{\tilde x}}$ with $ z \triangleq  \tilde x[x^{[K]}]$ such that
\begin{equation}
\begin{aligned}
 \tilde x[x^{[0]}] = x, \ \tilde X \succeq \tilde x \tilde x^\top, \ (\tilde X, \tilde x) \in \mathcal{N}^{[k]}_{\text{SDP}} ~ \forall k = \{1\dots K\} 
\end{aligned}
\label{eqn:sdp-relaxation}
\end{equation}
where the membership condition $(\tilde X,\tilde x)\in \mathcal{N}^{[k]}_\text{SDP}$ is defined by the following conditions:
\begin{equation}
	\begin{aligned}
			&\tilde x[x^{[k]}] \geq 0,\\
			&\tilde x[x^{[k]}] \geq W^{[k-1]}\tilde x[x^{[k-1]}], \\
			&\diag(\tilde X[x^{[k]}(x^{[k]})^T]) = \diag(W\tilde X[x^{[k-1]}(x^{[k]})^T]), \\
			&\diag(\tilde X[x^{[k-1]}(x^{[k-1]})^T]) \leq (l^{[k-1]}+u^{[k-1]}) \odot\\
			& \tilde X[x^{[k-1]}] -l^{[k-1]} \odot u^{[k-1]}, \\
	\end{aligned}
	\label{eq: one_layer_relu_constraint_sdp}
\end{equation}
where $n_{\tilde x} = \Sigma_{j=0}^K n_j$. $l^{[k]}$ and $u^{[k]}$ are the lower and upper bounds on $x^{[k]}$, respectively. Given $\mathcal{X}$, only $l^{[0]}$ and $u^{[0]}$ are known. The indexing notation in this paper is inherited from \cite{raghunathan2018semidefinite} to promote consistency; namely, $a[x^{[k]}] = a[\Sigma_{j=0}^{k-1} n_j+1:\Sigma_{j=0}^{k} n_j]$ for any vector $a \in \mathbb{R}^{n_{\tilde x}}$ and $A[x^{[k]}(x^{[l]})^T] = A[\Sigma_{j=0}^{k-1} n_j+1:\Sigma_{j=0}^{k} n_j, \Sigma_{j=0}^{l-1} n_j+1:\Sigma_{j=0}^{l} n_j]$ for any matrix $A \in  \mathbb{S}^{n_{\tilde x}}$.

For the sake of brevity, henceforth we use the shorthand notations $\tilde x[k] \triangleq \tilde x[x^{[k]}]$, $\tilde X[A_k] \triangleq \tilde X[x^{[k-1]}(x^{[k-1]})^T]$, $\tilde X[B_k] \triangleq \tilde X[x^{[k-1]}(x^{[k]})^T]$, $\tilde X[C_k] \triangleq \tilde X[x^{[k]}(x^{[k]})^T]$, and:
\begin{equation}
	\tilde X[k] \triangleq \begin{bmatrix}
		\tilde X[A_k] && \tilde X[B_k] \\
		\tilde X[B^\top_k] && \tilde X[C_k]
	\end{bmatrix}
\end{equation}

Furthermore, define:
\begin{equation}
	\Xi \triangleq \begin{bmatrix}
		1 && \tilde x^\top \\
		\tilde x && \tilde X
	\end{bmatrix}
\end{equation}
Note that $\tilde X \succeq \tilde x \tilde x^\top$ if and only if $\Xi \succeq 0$.

The above relaxation implies that if $\hat f_i^*(\mathcal{X}) \leq 0$, then it is guaranteed  that $f_i^*(\mathcal{X}) \leq 0$. However, if $\hat f_i^*(\mathcal{X}) \geq 0$, it is impossible to conclude whether $f_i^*(\mathcal{X}) \geq 0$ or the relaxation is loose. 

\subsection{Tightness of SDP Relaxation}
\label{sec:tightness}
Compared to previous convex relaxation schemes (such as LP), SDP indeed yields a tighter lower bound \cite{raghunathan2018semidefinite}. However, according to the above paper and the recent results in \cite{zhang2020tightness}, SDP relaxations of Multi-Layer Perceptron (MLP) ReLU networks are not tight even for single-layer instances. This issue will be elaborated below. 

Since $\Xi$ is positive-semidefinite, it can be decomposed as: 
\begin{equation}
	\Xi = V V^\top, \quad \text{where} \ V = \begin{bmatrix}
		\vec e \\ \vec x_1^\top \\ \vec x_2^\top \\ \vdots  \\ \vec x_{n_{\tilde x}}^\top 
	\end{bmatrix} \in \mathbb{R}^{(n_{\tilde x}+1) \times r}
	\label{eqn:gram_factorization}
\end{equation}
Each vector $\vec x_i$ has dimension r, which is the rank of $\Xi$. Since $\vec e \cdot \vec e = 1$, $\vec e$ is a unit vector in $\mathbb{R}^r$. Furthermore, we have $\tilde x_i = \vec e \cdot \vec x_i$ for $i \in \{1,\dots,\tilde n_x\}$. The constraints in $\mathcal{N}^{[k]}_{
\text{SDP}}$ can be broken down into 2 parts: Input constraints, and ReLU constraints. 

Input constraints can be regarded as restricting each vector in the set $\{\vec x_i | i \in \idx(\tilde x, \tilde x[k-1])\}$ to lie in a circle centered at $\frac{1}{2}(l_i+u_i)\vec e$ with radius $\frac{1}{2}(u_i-l_i)$. ReLU constraints, as a generalization to the analysis performed in the aforementioned papers, can be interpreted as $\vec \chi_j$ lying \emph{on} the circle with $\vec x_j$ as its diameter for all $\{j | j \in \idx(\tilde x, \tilde x[k])\}$. Here, $\vec \chi_j = \sum_{i=1}^{|x^{[k-1]}|} w_{ij} \vec x_i$, where $\{w_{ij}\}_{i=1}^{|x^{[k-1]}|}$ is the $j^{\text{th}}$ row of $W^{[k]}$. ReLU constraints also constrain $\vec x_j$ to have a nonnegative dot product with $\vec e$, and a longer projection on $\vec e$ than $\vec \chi_j$ does.

As pointed out in Lemma 6.1 of \cite{zhang2020tightness}, the SDP relaxation is tight if and only if $\vec e$ and all the vectors $\vec x_i$'s are collinear. Since there is no constraint on the angle between $\vec x_i$ and $\vec e$, the resulting spherical cap (as termed in Section 4 of the paper) always exists, and the height of this cap will be an upper bound on the relaxation gap of the corresponding entry in $\tilde x$.

\section{Convex Cuts}
As a well-known technique in nonlinear and mixed-integer optimization, adding valid convex constraints (convex cuts) could potentially reduce the relaxation gap of the problem. The most effective cut for SDP relaxation is based on the reformulation-linearization technique (RLT), which is obtained by multiplying linear constraints and relaxing the product into linear matrix constraints \cite{sherali2013reformulation}.

In the formulation of $\mathcal{N}^{[k]}_{\text{SDP}}$, there are only 2 linear constraints: $\tilde x[k] \geq 0 \ \text{and} \ \tilde x[k] \geq W^{[k-1]}\tilde x[k-1]$. The RLT cuts associated with these linear constraints exist in the original constraint set $\mathcal{N}^{[k]}_{\text{SDP}}$, except for the one obtained by the following multiplication:
\begin{equation}
\begin{aligned}
	& (\tilde x[k] - W^{[k-1]}\tilde x[k-1])(\tilde x[k] - W^{[k-1]}\tilde x[k])^\top \geq 0
\end{aligned}
\end{equation}

which leads to the relaxed cut:
\begin{equation}
	\begin{aligned}
		&\text{diag}(\tilde X[C_k] -  W^{[k-1]}X[B_k] - X[B_k^\top]W^{[k-1]^\top} \\
	& + W^{[k-1]}  X[A_k]  W^{[k-1]^\top}) \geq 0
	\end{aligned}
	\label{eqn:RLT}
\end{equation}

Note that only the diagonal entries are of interest because the other terms do not appear in the original QCQP formulation of the problem.

\subsection{Deficiency of RLT}
Although the above RLT reformulation seems to add new information to the SDP relaxation, a closer examination reveals otherwise. 
\begin{proposition}
	Constraint set (\ref{eq: one_layer_relu_constraint_sdp}) implies the RLT inequality (\ref{eqn:RLT}) for all $k \in \{1,\dots,K\}$.
\end{proposition}
\begin{proof}
	Since $\tilde X \succeq 0$, all principle submatrices $\tilde X[k]$ are positive semidefinite for $k \in \{1,\dots,K\}$. Then, $\tilde X[k] \succeq 0$ can be restated in terms of the general Schur's complement:
\begin{equation}
\begin{aligned}
	&\tilde X[k] \succeq 0 \Longleftrightarrow \{\tilde X[C_k] \succeq 0, \\
	&\tilde X[A_k] - \tilde X[B_k](\tilde X[C_k])^{\dagger}\tilde X[B^\top_k] \succeq 0 , \\
	&(I-\tilde X[C_k](\tilde X[C_k])^{\dagger})\tilde X[B^\top_k] = 0 \}
\end{aligned}
\end{equation}
Now, one can write:
\begin{equation}
\begin{aligned}
	&\text{diag}(W^{[k-1]}\tilde X[A_k]W^{[k-1]^\top} - \\
	&W^{[k-1]}\tilde X[B_k](\tilde X[C_k])^{\dagger}\tilde X[B^\top_k]W^{[k-1]^\top}) \geq 0
\end{aligned}
\end{equation}
After noticing that $\diag(\tilde X[C_k]) = \diag(W^{[k-1]}\tilde X[B_k])$ and $(I-\tilde X[C_k](\tilde X[C_k])^{\dagger})\tilde X[B^\top_k] = 0$, we obtain:
\begin{equation}
\begin{aligned}
	\text{diag}(W^{[k-1]}\tilde X[A_k]W^{[k-1]^\top} - \tilde X[B^\top_k]W^{[k-1]^\top}) \geq 0
\end{aligned}
\end{equation}
By adding $\diag(\tilde X[C_k]-W^{[k-1]}\tilde X[B_k]) = 0$ to both sides of equation, the above inequality yields (\ref{eqn:RLT}).
\end{proof}

Therefore, adding convex cuts to (\ref{eq: sdp_relaxed_certification_problem}) using RLT will only increase computation time without reducing the relaxation gap.

\section{Non-Convex Cuts}

\subsection{Source of Relaxation Gap}
The SDP relaxation is obtained by replacing the equality constraint $\tilde X-\tilde x \tilde x^\top=0$ with the convex inequality $\tilde X-\tilde x \tilde x^\top \succeq 0$. Hence, the non-convex constraint $\tilde X-\tilde x \tilde x^\top \preceq 0 $ excluded in this formulation is the source of the relaxation gap. 

One necessary condition for $\tilde X-\tilde x \tilde x^\top \preceq 0$ is:
\begin{equation}
	\begin{aligned}
		& \phi_i^\top (\tilde X - \tilde x \tilde x^\top) \phi_i \leq 0 ~ \forall i ~ \text{such that} \\
		& \{\phi_1, \dots, \phi_{n_{\tilde x}}\} \text{forms a basis in} ~ \mathbb{R}^{n_{\tilde x}}
	\end{aligned} \label{eqn:nsd_constraints}
\end{equation}
However, since $-\|\phi_i^\top \tilde x\|^2$ is concave in $\tilde x$, it is impossible to add this constraint under a convex optimization framework.

A non-convex cut is defined to be a valid constraint that is non-convex. In this case, any constraint in the form of (\ref{eqn:nsd_constraints}) is a non-convex cut. 

\subsection{A Penalization Approach}
\label{sec:penalization}
Since directly incorporating any non-convex cut makes the resulting problem non-convex, one may resort to penalization techniques. Explicitly, the objective of (\ref{eq: sdp_relaxed_certification_problem}) can be modified as:
\begin{equation*}
	c_i^\top z + \sum_{i=1}^{n_{\tilde x}} \phi_i^\top (\tilde x \tilde x^\top - \tilde X) \phi_i
\end{equation*}
Nevertheless, the constraint $\tilde X - \tilde x \tilde x^\top \succeq 0$ implies that $c_i^\top z + \sum_{i=1}^{n_{\tilde x}} \phi_i^\top (\tilde x \tilde x^\top - \tilde X) \phi_i \leq c_i^\top z$, therefore making the new problem not necessarily a relaxation of the original problem, i.e. the case $\hat f_i^*(\mathcal{X}) \leq f_i^*(\mathcal{X})$ is possible. Moreover, this penalization approach adds a convex term ($\|\phi_i^\top \tilde x\|^2$) to the maximizing objective, which destroys the convexity of the problem.

To avoid this issue, one may use the technique proposed in \cite{luo2019enhancing}. That work penalizes a given QCQP with a linear objective such that the problem remains a relaxation of the original problem after penalization. The authors proposed a sequential SDP procedure to approximate $ f_i^*(\mathcal{X})$ for any $i$ in (\ref{eq: robustness_certification_problem}) using the modified objective $p^*_i(\mathcal{X})$, defined as:
\begin{align}
	 p^*_i = \sup_{(x,z) \in\mathcal{N}_\text{SDP},\ x \in \mathcal{X}}c_i^\top z + \sqrt{c_i^\top (\tilde x \tilde x^\top - \tilde X)c_i}
	\label{eqn:penalized_objective}
\end{align}
it can be shown that $p^*_i(\mathcal{X})$ remains a relaxation of the original problem (\ref{eq: robustness_certification_problem}), i.e. $p^*_i(\mathcal{X}) \geq f_i^*(\mathcal{X})$.

\subsubsection{Deficiency of Penalty Methods}

The problem with the approach proposed in (\ref{eqn:penalized_objective}) is that the vector $c_i$ inside the square root must match that of the original linear objective, making it practically limited. Arguing under the framework of adding constraints in the form of (\ref{eqn:nsd_constraints}), the method in \cite{luo2019enhancing} only enforces one of them, namely $c_i^\top (\tilde X - \tilde x \tilde x^\top) c_i \leq 0$. Objective (\ref{eqn:penalized_objective}) ensures that $c_i^\top (\tilde X - \tilde x \tilde x^\top) c_i$ is as small as possible.

Although interior point methods are known to converge to maximum rank solutions, the relaxed matrix is low rank in general. Therefore, it is likely that the constraint $c_i^\top (\tilde X - \tilde x \tilde x^\top) c_i \leq 0$ is already satisfied for the unmodified solution. Simulations found in Section 6 corroborate this fact.

\subsection{Secant Approximation}
To address the above-mentioned issues, we leverage the method of secant approximation, inspired by \cite{saxena2010convex}. Note that the constraint (\ref{eqn:nsd_constraints}) can be written as:
\begin{equation*}
	 -(\phi_i^\top \tilde x)^2 \leq -\langle \tilde X, \phi_i \phi_i^\top \rangle
\end{equation*}

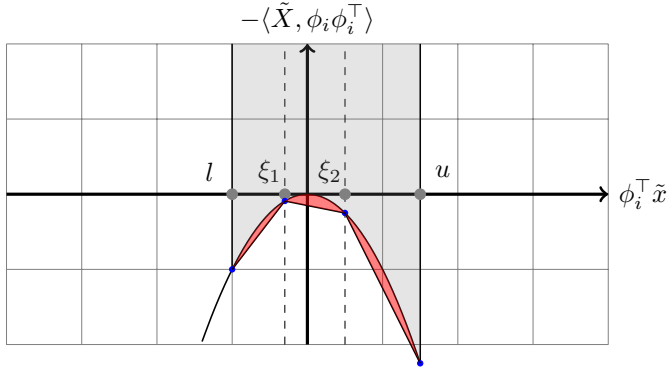
\begin{figure}[htp]
    \centering
    \begin{tikzpicture}
	\draw[step=1,gray,very thin] (-4,-2) grid (4,2);
  	\draw[->,very thick] (-4, 0) -- (4, 0) node[right] {$\phi_i^\top \tilde x$};
  	\draw[->,very thick] (0, -2) -- (0, 2) node[above] {$-\langle \tilde X, \phi_i \phi_i^\top \rangle$};
  \draw[domain=-1.4:1.5, smooth, variable=\x, black, semithick] plot ({\x}, {-\x*\x});
  \draw[semithick] (-1,-1)--(-1,2);
  \draw[semithick] (1.5,-2.25)--(1.5,2);
  \filldraw [gray] (-1,0) circle (2pt);
  \node[xshift=-3mm,yshift=3mm] at (-1,0)  {$l$};
  \filldraw [gray] (1.5,0) circle (2pt);
  \node[xshift=3mm,yshift=3mm] at (1.5,0)  {$u$};
  
  \draw[dashed] (0.5,-2) -- (0.5,2);
  \filldraw [gray] (0.5,0) circle (2pt);
  \node[xshift=-2mm,yshift=3mm] at (0.5,0)  {$\xi_2$};
  \draw[dashed] (-0.3,-2) -- (-0.3,2);
  \filldraw [gray] (-0.3,0) circle (2pt);
  \node[xshift=-2mm,yshift=3mm] at (-0.3,0)  {$\xi_1$};
  
  \fill [gray, opacity=0.2,domain=-1:1.5, variable=\x]
      (-1, 2)
      -- plot ({\x}, {-\x*\x})
      -- (1.5, 2)
      -- cycle;
      
  \filldraw [blue] (-1,-1) circle (1pt);
  \filldraw [blue] (-0.3,-0.09) circle (1pt);
  \filldraw [blue] (0.5,-0.25) circle (1pt);
  \filldraw [blue] (1.5, -2.25) circle (1pt);

  \draw[semithick] (-1, -1) -- (-0.3, -0.09);
  \draw[semithick] (-0.3, -0.09) -- (0.5, -0.25);
  \draw[semithick] (0.5, -0.25) -- (1.5, -2.25);
  
  \fill [red, opacity=0.5,domain=-1:-0.3, variable=\x]
      (-1, -1)
      -- plot ({\x}, {-\x*\x})
      -- (-0.3, -0.09)
      -- cycle;
      
  \fill [red, opacity=0.5,domain=-0.3:0.5, variable=\x]
      (-0.3, -0.09)
      -- plot ({\x}, {-\x*\x})
      -- (0.5, -0.25)
      -- cycle;
      
   \fill [red, opacity=0.5,domain=0.5:1.5, variable=\x]
      (0.5, -0.25)
      -- plot ({\x}, {-\x*\x})
      -- (1.5, -2.25)
      -- cycle;
  
	\end{tikzpicture}
    \caption{ $l$ and $u$ are lower and upper bounds of $\{\phi_i^\top \tilde x | ~ (\tilde x, \tilde X) \in \mathcal{N}_{\text{SDP}}\}$, respectively. Red areas indicate false feasible space induced by the secant approximation.}
    \label{fig: secant_approx}
\end{figure}
Graphically, the grey area in Figure \ref{fig: secant_approx} indicates the original feasible space of $(\tilde x, \tilde X)$ constrained by $-(\phi_i^\top \tilde x)^2 \leq -\langle \tilde X, \phi_i \phi_i^\top \rangle$. Since the set is non-convex, we use secant lines to approximate it. The approximated feasible space contains every point above the 3 secant lines, namely the grey area plus the red area. $\xi_1,\xi_2$ are the points by which the feasible space is divided. For each part of the space, we use a separate line to lower-bound it. In other words, for each divided space, a necessary condition for $(\tilde x, \tilde X)$ to lie in the original feasible space is given. As observed in Figure \ref{fig: secant_approx}, since secants are used for approximation, it is important that $l$ and $u$ be known in advance. For this certification problem, those bounds can be calculated with auxiliary SDPs. More importantly, for any $\{ (\tilde x, \tilde X) \in \mathcal{N}_{\text{SDP}} | \tilde X = \tilde x \tilde x^\top \}$, $(\tilde x, \tilde X)$ lies above exactly one secant line. 

In mathematical language, this means that
\begin{equation}
	\begin{aligned}
			 \bigvee_{q \in \{0,\dots ,Q+1\}} (\tilde x, \tilde X) \in \Gamma_q(\phi_i)
	\end{aligned} \label{eqn:disjuntive_secant}
\end{equation}
where $ \Gamma_q(\phi_i)$ is defined as
\begin{equation}
	\begin{aligned}
			\bigvee_{q \in \{0,\dots ,Q+1\}} \begin{bmatrix}
			 	\xi_q \leq \phi_i^\top \tilde x \leq \xi_{q+1}, ~ \text{and} \\
			 	\xi_q \xi_{q+1} - (\phi_i^\top \tilde x) (\xi_q + \xi_{q+1}) \leq -\langle \tilde X, \phi_i \phi_i^\top \rangle
			 \end{bmatrix}
	\end{aligned}
\end{equation}
and $\bigvee$ is a logical OR symbol commonly used in boolean algebra and $\xi_0 \triangleq l, \xi_{q+1} \triangleq u$, with $Q$ being the number of dividing points. This means that every pair $(\tilde x, \tilde X)$ in $\{ (\tilde x, \tilde X) \in \mathcal{N}_{\text{SDP}} | \tilde X = \tilde x \tilde x^\top \}$ belongs to only \emph{one} $\Gamma_q(\cdot)$ for some $q$, except when $\phi_i^\top \tilde x = \xi_q$ for $q= \{0,\dots, Q+1\}$.

\begin{lemma}
	The secant approximations (\ref{eqn:disjuntive_secant}) exactly recovers the constraint $-(\phi_i^\top \tilde x)^2 \leq -\langle \tilde X, \phi_i \phi_i^\top \rangle$ when $Q \rightarrow \infty$.
		\label{lemma:infinite_division}
\end{lemma}
\begin{proof}
	As $Q \rightarrow \infty$, $\xi_q = \phi_i^\top \tilde x = \xi_{q+1}$ for $q = \{0,\dots,Q+1\}$. Therefore, $\xi_q \xi_{q+1} - (\phi_i^\top \tilde x) (\xi_q + \xi_{q+1}) = (\phi_i^\top \tilde x)^2 - 2(\phi_i^\top \tilde x)^2 = -(\phi_i^\top \tilde x)^2$, resulting in $-(\phi_i^\top \tilde x)^2 \leq -\langle \tilde X, \phi_i \phi_i^\top \rangle$
\end{proof}
Of course, an infinite number of divisions is impractical from both a complexity standpoint and a numerical standpoint. Thus, selecting $\xi_q$'s intelligently to minimize the red area in Figure \ref{fig: secant_approx} is critical.
\begin{proposition}
	Given $l$ and $u$ for $\{\phi_i^\top \tilde x | ~ (\tilde x, \tilde X) \in \mathcal{N}_{\text{SDP}}\}$, a secant approximation with $Q$ division points $\{\xi_q\}_{q=1}^Q$ achieves the best approximation when $[l,u]$ is equally partitioned.
	\label{prop:equal_division}
\end{proposition}
\begin{proof}
	We optimize for the red area in Figure \ref{fig: secant_approx}:
	\begin{equation*}
		\min_{\{\xi_q\}_{q=1}^Q} \sum_{q=1}^{Q+1} \int_{\xi_{q-1}}^{\xi_q} -x^2 - (\xi_{q-1} \xi_q - x(\xi_{q-1}+\xi_q)) d x 
	\end{equation*}
	Since the objective $f(\cdot)$ is convex in $\{\xi_q\}_{q=1}^Q$, $\nabla f(\{\xi_q\}_{q=1}^Q) = 0$ implies that $\xi_q = \frac{1}{2}(\xi_{q-1}+\xi_{q+1})$ for $q= \{0,\dots, Q+1\}$. This further implies an equal partitioning of $[l,u]$.
\end{proof}

\section{Sequential Construction of Valid Cuts}

Based on the discussion of the pros and cons of various strengthening techniques in the previous section, the method of secant approximation offers a significant benefit over the existing techniques. However, we have only discussed how to approximate a single constraint $\phi_i^\top (\tilde X - \tilde x \tilde x^\top) \phi_i > 0$ for a particular $\phi_i$ vector, but the deployment of secant approximation requires an in-depth analysis and careful design.

There are 2 main questions arising from this approach:
\begin{enumerate}
	\item Which $\phi_i$ vector should be chosen? How do we know if the constraint of $\phi_i^\top (\tilde X - \tilde x \tilde x^\top) \phi_i > 0$ is already satisfied?
	\item How many such constraints shall be added?
\end{enumerate}
The first question can be addressed via a technique called Cut Generating Linear Programming (CGLP) and the second one can be solved via a sequential procedure. The two above-mentioned solutions will be elaborated in the following two subsections.

\subsection{Constructing Valid Cuts}
To study whether $\phi_i^\top (\tilde X - \tilde x \tilde x^\top) \phi_i > 0$ is already satisfied, a candidate solution $(\tilde x^*, \tilde X^*)$ is needed, which can be easily obtained by running the original SDP problem. By direct substitution, one can verify if the given constraint $\phi_i^\top (\tilde X^* - \tilde x^* \tilde (x^*)^\top) \phi_i > 0$ is redundant. More importantly, with the candidate solution, one can find the eigenvectors corresponding to the largest positive eigenvalues of $\tilde x^*(\tilde x^*)^\top- \tilde X^*$ in order to find $\phi_i$s that violate the negative semidefinite constraint the most.

However, this is not enough, because secant approximation induces a false feasible space (colored in red in Figure \ref{fig: secant_approx}). As a result, it could happen that while a given constraint is not redundant, its secant approximation is. This phenomenon is well illustrated by Example 1 in \cite{saxena2010convex}.

Therefore, it is necessary to actively search for a non-redundant secant approximation, termed a \emph{valid cut}, given a candidate $\phi_i$ vector, which is usually an eigenvector of $\tilde x^*(\tilde x^*)^\top- \tilde X^*$ corresponding to a strictly positive eigenvalue. Mathematically, this can be accomplished via the Cut Generating Linear Programming (CGLP) problem \cite{saxena2010convex}:

\begin{equation}
	\tag{CGLP}
	\begin{aligned}
		\min_{\alpha,\beta, \{\mu^q\}, \{\nu^q\}} ~ &\alpha^\top \chi^* -\beta \\
		\text{s.t.} ~ &A^\top \mu^q + D_q^\top \nu^q \leq \alpha \quad q \in \{1,\dots, Q\} \\
		&b^\top \mu^q + d_q^\top \nu^q \geq \beta \quad q \in \{1,\dots, Q\} \\
		&\mu^q, \nu^q \geq 0 \quad q \in \{1,\dots, Q\} \\
		& \sum_{q=1}^Q (\kappa^\top \mu^q + \eta_q^\top \nu^q) = 1
	\end{aligned}
	\label{eqn:cglp}
\end{equation}
where $\chi$ is a vectorized version of the variables $(\tilde x, \tilde X)$ with $\chi^* = [\tilde x^{*^\top}, \vecc(\tilde X^*)^\top]^\top$, $\{\mu^q\}, \{\nu^q\}$ being vectors of nonnegative entries, and $\kappa, \{\eta_q\}$ being normalizing constants. $\kappa, \{\eta_q\}$ are some constants to help with numerical stabilities, and are not of theoretical interest.

Most importantly, $A \chi \geq b$ ($A \chi$ is just matrix multiplication) is a reformulation of the constraint $(\tilde x, \tilde X) \in \mathcal{N}_{\text{SDP}}$, and $D_q^\top \chi \geq d_q$ is a reformulation of the constraint $(\tilde x, \tilde X) \in \Gamma_q(\phi_i)$ for any $\phi_i$. This is possible because every constraint is affine after lifting (i.e. introducing $\tilde X$ in the place of $\tilde x \tilde x^\top$). That is, different CGLPs exist for different $\phi_i$s. Moreover, note that given $Q \in \mathbb{Z}$, $\{\xi_q\}_{q=1}^Q$ are calculated according to Proposition \ref{prop:equal_division}.

To illustrate how CGLP can help, we introduce the following theorem, which is a modification of Theorem 1 in \cite{saxena2010convex}:
\begin{theorem}
	\label{thm:valid_cut}
	If the optimal value of (\ref{eqn:cglp}) is negative, then a valid cut $\alpha \chi \geq \beta$ is found that cuts off the candidate solution $(\tilde x^*, \tilde X^*)$. Conversely, if the optimal value is nonnegative, then no such cut exists for $\{\Gamma_q(\phi_i)\}_{q=0}^{Q+1}$.
\end{theorem}
\begin{proof}
	According to Theorem 3.1 of \cite{balas1998disjunctive}, $\alpha \chi \geq \beta$ is a valid constraint for (\ref{eq: sdp_relaxed_certification_problem}) with the secant approximations (\ref{eqn:disjuntive_secant}) if and only if the first three lines of the \ref{eqn:cglp} constraints hold true. All alphas and betas for which $\alpha \chi \geq \beta$ is a valid constraint is in the polyhedral search space of CGLP. If there exist $\alpha^*$ and $\beta^*$ such that $\alpha^* \chi^* < \beta^*$, the optimal value of CGLP must be negative. On the other hand, if the optimal value of CGLP is negative, such $\alpha^*$ and $\beta^*$ must exist. In this case, $\alpha^*$ and $\beta^*$ must also meet the condition $\alpha^* \chi \geq \beta^*$ due to the convexity of polyhedra. Then, $\chi^*$ contradicts the requirements of $\chi$, making the constraint $\alpha \chi \geq \beta$ a valid cut.
	
	Conversely, if no such $\alpha^*$ and $\beta^*$ exist, then $\chi^*$ already satisfies all valid constraints with respect to $\{\Gamma_q(\phi_i)\}_{q=0}^{Q+1}$ since the polyhedral search space is convex, and convexity guarantees an exhaustive search. Thus, no valid cut can be found.
\end{proof}

Theorem \ref{thm:valid_cut} states that if the CGLP problem for a particular $\phi_i$ vector returns a negative optimum, then $\alpha \chi \geq \beta$ is a \emph{valid cut}, and can be added to the SDP problem to further strengthen the problem.

Theorem \ref{thm:valid_cut} and Lemma \ref{lemma:infinite_division} lead to the following  result:
\begin{corollary}
	If ${D_k^\top \chi \geq d_k}$ corresponds to a set of constraints $\bigvee_{q \in \{0,\dots ,Q+1\}} (\tilde x, \tilde X) \in \Gamma_q(\phi_i)$ with $\phi_i$ being an eigenvector of $\tilde x^* \tilde x^{*\top}- \tilde X^*$ associated with a positive eigenvalue, then a valid cut $\alpha \chi \geq \beta$ always exists if $Q \rightarrow \infty$.
	\label{corollary:1}
\end{corollary}


\subsection{A Sequential Algorithm}
To systematically generate constraints of the form $\phi_i^\top (\tilde X - \tilde x \tilde x^\top) \phi_i > 0$, we propose Algorithm \ref{alg:main} and study its performance in this part.

\begin{algorithm} [ht]
    \SetKwInOut{Input}{Input}
    \SetKwInOut{Output}{Output}

    \underline{verification} $(\mathcal{X},\mathcal{S},\{W^{0} \dots W^{K-1}\},Q,\text{max}_{\text{iter}},\gamma)$\;
    \Input{Input uncertainty set $\mathcal{X}$, safe set $\mathcal{S}$, trained network weights $\{W^{0} \dots W^{K-1}\}$, number of partitions for the secant approximation, maximum number of iterations $\max_{iter}$, and the eigenvalue threshold $\gamma$.}
    \Output{$\tau_r = \max {C_r^\top z} $ for $r \in \{1, \dots, R\}$, where $C_r$ is the $r^{\text{th}}$ row of $C$, and R is the number of rows of $C$, as specified by the safe set $\mathcal{S}$}
    \For{$c = C_1, \dots, C_R$ }{
    $(\tilde x^*, \tilde X^*) \leftarrow \argmax\{c^\top z : (x,z) \in\mathcal{N}_\text{SDP}, ~ x \in\mathcal{X} \}$ \;
    $\tilde \alpha = [], \tilde \beta = []$ \;
    \For{$i = 1,\dots, \text{max}_{\text{iter}}$ }{
    $(\lambda \in \mathbb{R}^{m},V \in \mathbb{R}^{n_{\tilde x} \times m}) \leftarrow$ eigenvalues of $\tilde x^* \tilde x^{*^\top}-\tilde X^*$ bigger than $\gamma$, and with eigenvector corresponding to the $i^{\text{th}}$ eigenvalue being $i^{\text{th}}$ column of $V$ \;
    \For{$\phi = V_1, \dots, V_m$}{
    $(l,u) \leftarrow \{\min,\max\} \{\phi^\top z : (x,z) \in\mathcal{N}_\text{SDP}, ~ x \in\mathcal{X} \}$ \;
    $(\alpha,\beta) \leftarrow$ CGLP$(Q,\phi,u,l)$\;
    \If{$\alpha^\top \chi^* < \beta$}
      {
        $\tilde \alpha = [\tilde \alpha; \alpha^\top], \tilde \beta = [\tilde \beta; \beta]$
      }
	}
	$(\tilde x^*, \tilde X^*) \leftarrow \argmax\{c^\top z : (x,z) \in\mathcal{N}_\text{SDP}, \tilde \alpha \chi \geq \tilde \beta, ~ x \in\mathcal{X} \}$ \;
	}
	$\tau_r = c^\top \tilde x^*[K]$\;}
    \caption{Sequential SDP verification of NN Robustness by adding secant-approximated non-convex cuts}
    	\label{alg:main}
    \end{algorithm}
    
    
The main result of this paper is stated below.
\begin{theorem}
	For every iteration $i \in \{2,\dots, \max_{\text{iter}}\}$ in Algorithm \ref{alg:main}, it holds that
	\begin{equation}
		f^*(\mathcal{X}) \leq c^\top \tilde x^*_{i}[K] \leq c^\top \tilde x^*_{i-1}[K] \leq \hat f^*(\mathcal{X})
	\end{equation}
	with $\tilde x^*_{i}, \tilde X^*_i$, and by extention $\Xi^*_i$ being the optimizers of the $i^{\text{th}}$ iteration. The middle inequality becomes strict if at least one of the optimal values of CGLPs in iteration $i$ is negative and either of the following holds:
	\begin{itemize}[noitemsep,topsep=0pt]
		\item $\hat c$ lies in the null space of $\Xi_{i-1}$, where $\hat c \triangleq [0_{1 \times (n_{\tilde x}-n_{K})} \ c^\top]^\top$
		\item $Q \rightarrow \infty$
	\end{itemize}
	\label{thm:main}
\end{theorem}
\begin{proof}
	The inequalities $c^\top \tilde x^*_{i}[K] \leq c^\top \tilde x^*_{i-1}[K] \leq \hat f^*(\mathcal{X})$ are due to the fact that a strict reduction in the search space will not yield a higher optimal value. $f^*(\mathcal{X})$ serves as a lower bound since the constraints are valid, according to Theorem \ref{thm:valid_cut}.
	
	Now, since the objective is linear, $\hat c^\top \tilde x = \gamma$ must be a supporting hyperplane for the spectrahedral (convex body arising from linear matrix inequalities) feasible set at $\tilde x_{i-1}^*$, where $\gamma$ denotes as the constant $c^\top \tilde x_{i-1}^*[K]$. By Proposition 2 in \cite{roshchina2017face}, the intersection of a closed convex set (e.g., this feasible set) and the supporting hyperplane is an exposed face. Furthermore, by \cite{ramana1995some}, it is known that any exposed face of a spectraheron is a proper face, and the null space of $\Xi$ will stay constant over the relative interior of this face. Thus, the intersection will consist of points $(\tilde x, \tilde X)$ such that $\hat c^\top \tilde x = \gamma$ and $\mathcal{N}(\Xi) = \mathcal{N}(\Xi_{i-1})$. 
	
	Let a basis of $\mathcal{N}(\Xi_{i-1})$ be denoted as $\{\psi_j\}$, where each vector is partitioned as $\psi_j \triangleq [\omega_j \ \ \hat \psi_j^\top]^\top$. Therefore, $\hat \psi_j^\top \tilde x = -\omega_j$ for $j \in \{1,\dots, n_{\tilde x}+1\}$. If $\hat c \in \mathcal{N}(\Xi_{i-1})$, the supporting hyperplane will intersect the spectrahedron at exactly one point (the face is of affine dimension 0) because $\tilde x$ is already fixed in $\{\hat \psi_j\}$ coordinates, and $\gamma$ can only take on one value. If the intersection only consists of one point, then a valid cut will invalidate this point, and the objective value will strictly decrease.
	
	On the other hand, if $Q \rightarrow \infty$, the original negative semidefinite constraint is recovered, as per Lemma \ref{lemma:infinite_division}. Then, after using the valid cut, $\Xi_i$ will have a strictly lower rank than $\Xi_{i-1}$. Since the intersection consists of points such that $\mathcal{N}(\Xi) = \mathcal{N}(\Xi_{i-1})$, all those points will not be in the feasible space anymore under the presence of this valid cut. This leads to a strict decrease in objective value.
\end{proof}
Since $n_{\tilde x}$ is usually very large for multilayer networks, the first condition is often satisfied. However, it is important to note that those 2 conditions are only sufficient conditions, and that in practice there is normally a non-zero reduction in relaxation gap whenever a valid cut is calculated. Moreover, if a valid cut is given, the algorithm will always reach new optimal points, thus avoiding the situation of repeating the same search again and again.

This algorithm is asymptotically exact if an infinite number of iterations is taken, as explained below.

\begin{lemma}
	The relation $c^\top \tilde x^*_{i}[K] = f^*(\mathcal{X})$ holds as $i \rightarrow \infty, $ provided that $\gamma = 0$ and $Q \rightarrow \infty$
	\label{lemma:asymptotic_recovery}
	\end{lemma}
\begin{proof}
	The proof follows directly from Theorems \ref{thm:valid_cut}, \ref{thm:main}, and Corollary \ref{corollary:1}.
\end{proof}

\section{Experiments}

We certify the robustness of the IRIS dataset%
\footnote{https://archive.ics.uci.edu/ml/datasets/iris} with pre-trained MLP classification network with 99\% accuracy on test data using the original SDP approach, our Algorithm \ref{alg:main}, and the penalization approach in Section \ref{sec:penalization}. For all experiments, $\text{max}_{\text{iter}}$ is set to 10. Moreover, the $l_{\infty}$ radius of $\mathcal{X}$ is set to 0.15 for networks with 5 and 10 hidden layers, and set to 0.075 for the other 2 networks. This is because larger networks are naturally more sensitive to perturbation.

\subsection{Certification Percentage}

\vspace{2mm}

\begin{table}[ht]
\centering
\caption{Certification Percentage (First Row) and Average Trace Gap (Second Row).}
\label{table:summary_percent}
\begin{tabular}{cccc}

\multicolumn{1}{l}{H-Layers/Q} & \multicolumn{1}{l}{SDP} & \multicolumn{1}{l}{Algorithm \ref{alg:main}} & \multicolumn{1}{l}{Penalized} \\ \hline
5, Q=20                          & 100\%                     & 100\%                      & 100\%                           \\ (Trace Gap)
                           & $4.73 \times 10^{-7}$   & $4.73 \times 10^{-7}$    & 0.48                          \\ 
10, Q=5                          & 0\%                     & 80\%                      & 0\%                           \\(Trace Gap)
                           & 31.15   & 27.09    & 20.42                        \\ 
15, Q=5                          & 0\%                     & 60\%                      & 0\%                           \\(Trace Gap)
                           & 491.71   & 61.1    & 228.68                          \\ 
15, Q=20                         & 0\%                     & 100\%                      & 0\%                           \\(Trace Gap)
                           & 513.71   & 70.05    & 198.79                        
\end{tabular}

\end{table}

In Table \ref{table:summary_percent}, certification percentage and average trace gaps ($\tr(\tilde X)-\tilde x^\top \tilde x$) for the original SDP procedure, our Algorithm \ref{alg:main}, and the penalized approach are all calculated for classification networks with different numbers of hidden layers. Trace gap is an alternative measure for the rank of $\tilde X$, since many non-zero eigenvalues can be really small, and it is not obvious whether we should count them when calculating rank.

It can be observed that when the network is small, the original SDP procedure is already sufficient. However, as the number of hidden layers grows, Algorithm \ref{alg:main} shows its strength by certifying more data points. The penalization approach remains ineffective although sometimes it decreases the trace gap. 

It should be noted that when a small Q cannot generate enough valid cuts (in this case Q=5), increasing that partition number (Q=20) will generally improve the performance, as can be noticed in the 2 cases with 15 hidden layers. This is also an empirical verification of Lemma~\ref{lemma:asymptotic_recovery}.

\begin{figure}[ht]
    \centering
    \begin{subfigure}{0.24\textwidth}
    	    \includegraphics[width=\linewidth,height=0.7\linewidth]{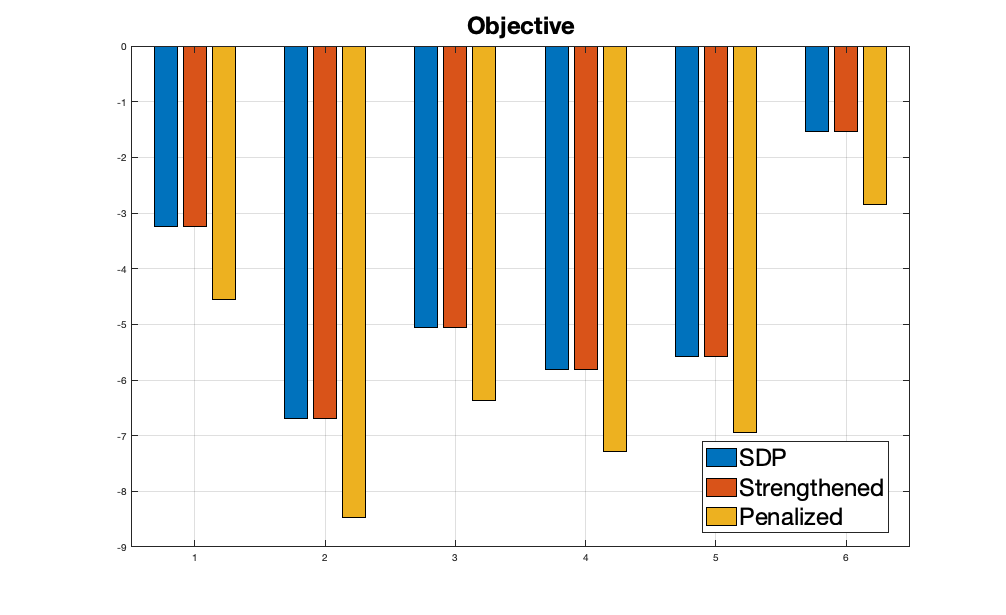}
    \caption{5 H-Layers, Q=20}
    \end{subfigure} \hfill
    \begin{subfigure}{0.24\textwidth}
    	    \includegraphics[width=\linewidth,height=0.7\linewidth]{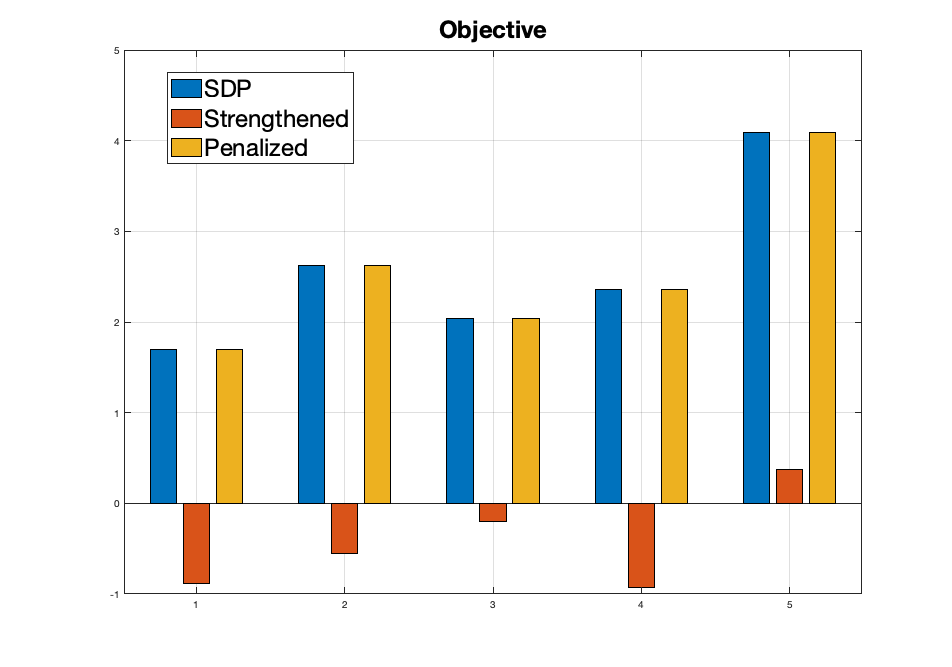}
    \caption{10 H-Layers, Q=5}
    \label{subfig:10_h}
    \end{subfigure}
    \vskip\baselineskip
    \begin{subfigure}{0.24\textwidth}
    	    \includegraphics[width=\linewidth,height=0.7\linewidth]{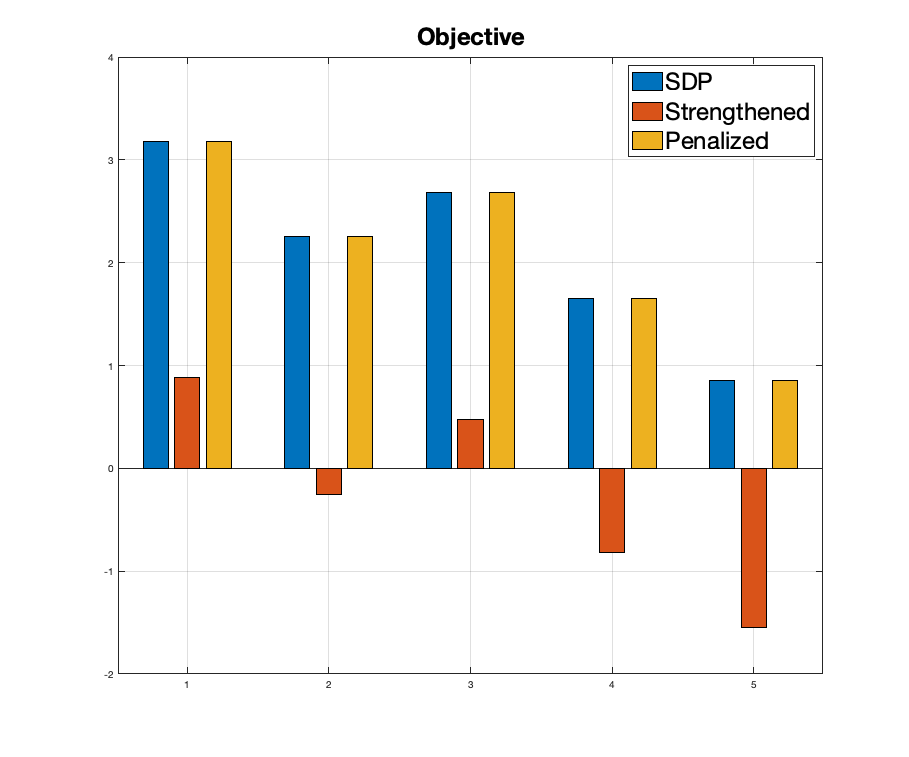}
    \caption{15 H-Layers, Q=5}
    \end{subfigure}\hfill
    \begin{subfigure}{0.24\textwidth}
    	    \includegraphics[width=\linewidth,height=0.7\linewidth]{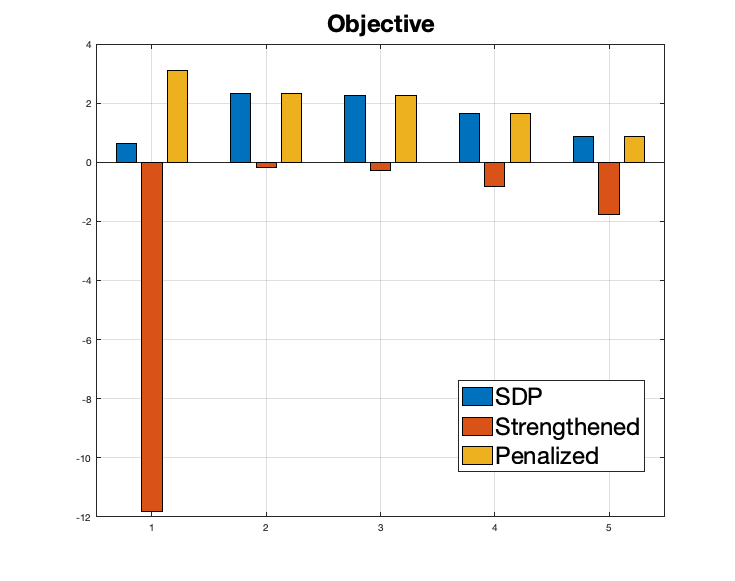}
    \caption{15 H-Layers, Q=20}
    \end{subfigure}
	\caption{$\hat f_i^*(\mathcal{X})$ for different certification methods and different networks}
	\label{fig:obj}
\end{figure}

The objective $\hat f_i^*(\mathcal{X})$ of (\ref{eq: sdp_relaxed_certification_problem}) for the aforementioned approaches is shown in Figure \ref{fig:obj} for 5 data points each.

\subsection{Running time}
The running time of Algorithm \ref{alg:main} is problem-specific, because the number of constraints that needs to be added is dependent on how loose the candidate solution is, and it is apparent from Figure \ref{fig:obj} that even for similar points in a small-scale dataset, the differences are large.

Most importantly, the most time consuming part of this algorithm is CGLP, because after linearization into vectors, the dimension of the LP is large. However, since the LP is sparse, one may use specialized methods to handle the LP in a more time-efficient fashion. 

Below is an illustration of the runtime of Algorithm \ref{alg:main} for the 5 datapoints shown in Figure \ref{subfig:10_h}. It is apparent that the CGLP procedure takes well over half of the time in the entire algorithm. 

However, since LP naturally scales much better than SDP and the complexity can be as low as linear in the size of the problem in presence of favorable sparsity, this algorithm can be easily scaled to larger datasets if specialized solvers of CGLP can be developed to handle the special structures.

Although Algorithm \ref{alg:main} is an exponential time algorithm in the worst case, as most of the proofs only work when $Q \rightarrow \infty$, the previous section has demonstrated that we can achieve satisfactory results even with small values of $Q$, making it also empirically appealing.  

\begin{table}[ht]
\caption{Running time of Algorithm \ref{alg:main} with 10 H-Layers.}
\begin{tabular}{ccc}
\multicolumn{1}{l}{Datapoint} & \multicolumn{1}{l}{Algorithm \ref{alg:main} runtime (s)} & \multicolumn{1}{l}{CGLP runtime (s)} \\ \hline
1                          & 166                     & 100   \\
2                          & 149                     & 88   \\ 
3                          & 143                     & 85   \\                                        
4                          & 149                     & 90   \\                                        
5                          & 165                     & 99   
                                                                                        
\end{tabular}
\label{table:summary}
\end{table}
%

\section{Geometric Analysis of Non-convex Cuts}
\label{sec:geometric_analysis}
In this section, we offer a geometric intuition into the success of non-convex cuts. We revisit (\ref{eqn:disjuntive_secant}) and write it as:
\begin{equation}
	\begin{aligned}
	\xi_q \xi_{q+1} - (\phi^\top \tilde x) (\xi_q + \xi_{q+1}) \leq -\langle \tilde X, \phi \phi^\top \rangle  \\
	\Leftrightarrow \langle \tilde X, \phi \phi^\top \rangle \leq (\phi^\top \tilde x -\xi_q)(\xi_{q+1}-\phi^\top \tilde x) + (\phi^\top \tilde x)^2
	\label{eqn:non-convex_cut}
\end{aligned}
\end{equation}

Since there always exists a partition number $Q$ large enough such that $(\phi^\top \tilde x -\xi_q)(\xi_{q+1}-\phi^\top \tilde x) \leq \epsilon$, it is possible to rewrite equation (\ref{eqn:non-convex_cut}):
\begin{equation}
\begin{aligned}
	\sum_{i,j \in \{n_{\tilde x}\} \times \{n_{\tilde x}\}, i \neq j} (\phi^i \phi^j)\|\vec x_i \| \|\vec x_j\| cos(\theta_{ij}) \leq \\
	\sum_{i,j \in \{n_{\tilde x}\} \times \{n_{\tilde x}\}, i \neq j} (\phi^i \phi^j)\|\vec x_i \| \|\vec x_j\| cos(\theta_i)cos(\theta_j) + \epsilon
\end{aligned}
\end{equation}
after substituting equation (\ref{eqn:gram_factorization}). Here, $\phi^i$ is the $i^{\text{th}}$ entry of $\phi$ (to be distinguished from $\phi_i$, which is a vector in the set of basis), and $\theta_{ij}$ is the angle between the vectors $\vec x_i, \vec x_j$, and $\theta_i$ denotes the angle between $\vec e$ and $\vec x_i$ for $i \in \{1,\dots,n_{\tilde x}\}$.

Therefore, there must exist a pair $(i,j)$ such that
\begin{equation}
\begin{aligned}
 cos(\theta_{ij}) \leq cos(\theta_i)cos(\theta_j) + \frac{\epsilon}{n_{\tilde x}}
\end{aligned}
\label{eqn:angle_ineq}
\end{equation}
Since $n_{\tilde x}$ is usually large in multi-layer networks and $\epsilon$ is chosen to be very small, $\frac{\epsilon}{n_{\tilde x}}$ can be neglected for practical purposes.

Recall from Section \ref{sec:tightness} that $\theta_i = 0 \ \forall i \in \{1,\dots,n_{\tilde x}\}$ is both sufficient and necessary for the tightness of the relaxation. Furthermore, in the following lemma we will show that any decrease in $\theta_i$ for any $i$ can contribute to a tighter relaxation.

\begin{lemma}
\label{lemma:decrease_relax_bound}
	For any row vector $\vec x$ in $V$, as part of the factorization of $\tilde X$, any increase in the lower bound for $\frac{\vec x \cdot \vec e}{\|\vec x\|}$ will decrease the upper bound on the relaxation gap for at least one entry in $\tilde x$. 
\end{lemma}
\begin{proof}
	Extending the results from Section \ref{sec:tightness}, it can be verified   that for a fixed $\theta$, the angle between $\vec e$ and $\vec \chi$, a decrease in $\|\vec \chi\|$ will lead to a smaller spherical cap. Now, consider a fixed $\|\vec \chi\|$. It can be shown that the height of the spherical cap is $\frac{\|\vec \chi\|}{2}(1-cos(\theta))$. Therefore, any increase in $cos(\theta)$ will lead to a smaller cap. Since $\vec \chi$ is just a linear combination of different $\vec x$, increasing the lower bound for $\frac{\vec x \cdot \vec e}{\|\vec x\|}$ will also increase the lower bound for $\frac{\vec \chi \cdot \vec e}{\|\vec \chi\|}$.
\end{proof}

Given any such pair $(i,j)$, if $cos(\theta_{ij})$ is large, namely possibly close to one, the term $cos(\theta_i)cos(\theta_j)$ will have to be larger under the constraint $cos(\theta_{ij}) \leq cos(\theta_i)cos(\theta_j) + \frac{\epsilon}{n_{\tilde x}}$. This means that $cos(\theta_i) \approx 1$ and $cos(\theta_j) \approx 1$, making both angles very small. Without this constraint, it is possible for $\vec x_j, \vec x_i$ to be collinear, but not collinear with $\vec e$, thus making them have large angles $\theta_i$ and $\theta_j$. Figure \ref{fig: different_angles} graphically showcases this difference.
\begin{figure}[ht]
    \centering
    \begin{subfigure}{0.2\textwidth}
    	\begin{tikzpicture}[thick,scale=1, every node/.style={scale=1}]
  	\draw[->,very thick] (-1.5, 0) -- (1.5, 0) ;
  	\draw[->,very thick] (0, -0.5) -- (0, 2) node[above] {$\vec e$};
  	\coordinate (origin) at (0,0);
  	\coordinate (i) at (1,2);
  	\coordinate (j) at (-1,1.5);
  	\coordinate (up) at (0,2);
  \draw[->,semithick] (0,0)--(1,2) node[right] {$\vec x_i$};
  \draw[->,semithick] (0,0)--(-1,1.5) node[left] {$\vec x_j$};
   \pic [draw, <->, "$\theta_i$", angle eccentricity=1.5] {angle = i--origin--up};
   \pic [draw, <->, "$\theta_j$", angle eccentricity=1.5] {angle = up--origin--j};
   \pic [draw, <->, "$\theta_{ij}$", angle radius = 1.2cm, angle eccentricity=1.3] {angle = i--origin--j};
  
	\end{tikzpicture}
	\caption{\scalebox{0.8}{$cos(\theta_{ij}) \leq cos(\theta_i)cos(\theta_j)$}}
\end{subfigure} \hspace{\textwidth}
\begin{subfigure}{0.2\textwidth}
    	\begin{tikzpicture}[thick,scale=1, every node/.style={scale=1}]
  	\draw[->,very thick] (-1.5, 0) -- (1.5, 0) ;
  	\draw[->,very thick] (0, -0.5) -- (0, 2) node[above] {$\vec e$};
  	\coordinate (origin) at (0,0);
  	\coordinate (i) at (1,1.5);
  	\coordinate (j) at (2,0.2);
  	\coordinate (up) at (0,2);
  \draw[->,semithick] (0,0)--(1,1.5) node[right] {$\vec x_j$};
  \draw[->,semithick] (0,0)--(2,0.2) node[right] {$\vec x_i$};
   \pic [draw, <->, "$\theta_j$", angle eccentricity=1.5] {angle = i--origin--up};
   \pic [draw, <->, "$\theta_i$", angle radius = 1.3cm, angle eccentricity=1.2] {angle = j--origin--up};
   \pic [draw, <->, "$\theta_{ij}$", angle radius = 0.8cm, angle eccentricity=1.3] {angle = j--origin--i};
  
	\end{tikzpicture}
	\caption{\scalebox{0.8}{$cos(\theta_{ij}) \geq cos(\theta_i)cos(\theta_j)$}}
\end{subfigure}
    \caption{Illustration of possible configurations of $\vec x_i$, $\vec x_j$ given a fixed $\theta_{ij}$}
    \label{fig: different_angles}
\end{figure}
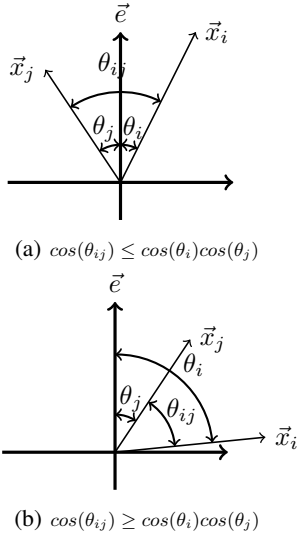
In particular, since ReLU only outputs positive values, we have $\theta_i \in [0,\frac{\pi}{2}]$ for all $i \in \{1,\dots,n_{\tilde x}\}$. However, the vectors $\vec x_i$ that correspond to the input layer are exceptions. Without loss of generality, it is possible to also assume these to be nonnegative by changing the weight matrix of the first layer. Therefore, $cos(\theta_i) \geq 0 \ \forall i \in \{1,\dots,n_{\tilde x}\}$. Thus, when $cos(\theta_i)cos(\theta_j) \approx 1$, we have $\theta_i \approx 0$ and $\theta_j \approx 0$. Without this constraint, it is possible that $\theta_j \approx \pi$ and $\theta_i \approx \pi$. In other words, with the ReLU constraints in place, the non-convex cuts will push the vectors towards $\vec e$ instead of $-\vec e$.

If $cos(\theta_{ij})$ is small, then inequality (\ref{eqn:angle_ineq}) will not be binding, and the original semidefinite constraints will work satisfactorily. Ideally, inequality (\ref{eqn:angle_ineq}) will hold for all pairs $(i,j)$, but this is not guaranteed. However, under the assumption that the weight matrices in the neural network are relatively well conditioned (no large condition number) and that $\epsilon$ is negligible, a large number of the pairs $(i,j)$ are expected to satisfy or approximately satisfy inequality (\ref{eqn:angle_ineq}). 

As explained, only those pairs $(i,j)$ with a relatively large $cos(\theta_{ij})$ are of interest. Denote the threshold as $\rho \in [0.5,1)$ such that $cos(\theta_{ij}) \geq \rho$. We call a pair $(i,j)$ to be valid if it satisfies inequality $cos(\theta_{ij}) \leq cos(\theta_i)cos(\theta_j) + \frac{\epsilon}{n_{\tilde x}}$ and invalid otherwise. In the worst case, $cos(\theta_i)cos(\theta_j)-cos(\theta_{ij}) = 1 -\rho$ for all valid pairs and $cos(\theta_i)cos(\theta_j)-cos(\theta_{ij}) = -\varrho < 0$ for all invalid pairs. This leads to the case where the number of valid pairs  is the smallest (denoted as $\vartheta_1$) and the number of invalid pairs is the largest (denoted as $\vartheta_2$). The other threshold $\varrho$ exists because a number $cos(\theta_{ij})$ too close to $cos(\theta_i)cos(\theta_j)$ approximately satisfies the inequality. Under the assumption that weight matrices are well conditioned, it implies that $\|\vec x_i \| \|\vec x_j\|$s are approximately the same amongst all $(i,j)$ pairs. Thus, if $\epsilon$ is negligible:
	\begin{equation}
	\begin{aligned}
	\sum_{i,j} (\phi^i \phi^j)\|\vec x_i \| \|\vec x_j\| (cos(\theta_i)cos(\theta_j)-cos(\theta_{ij})) \geq 0 \\
	\implies \vartheta_1(1-\rho) - \vartheta_2(\varrho) \approx 0 \implies \frac{\vartheta_1}{\vartheta_2} \approx \frac{\varrho}{1-\rho}
	\end{aligned}
	\end{equation}
	If $\rho=0.8$ and $\varrho = 0.2$, then at least half of all pairs such that $cos(\theta_{ij}) \geq \rho$ will be valid. As $\rho$ increases, so does $\frac{\vartheta_1}{\vartheta_2}$. This means that almost all of the most important pairs $(i,j)$ are valid.

Furthermore, the nature of a multi-layer setup means that if $\vec x_i$ and $\vec x_j$ are from different layers of the network, they depend on each other. Figure \ref{fig: different_angles} shows that the non-convex cut tends to make $\vec x_i$ and $\vec x_j$ lie on different sides of $\vec e$ in the event of a large $cos(\theta_{ij})$. This means that as vectors build on each other (see Figure 1 in \cite{raghunathan2018semidefinite}), they will revolve around $\vec e$, instead of branching out into a certain direction. This further implies a small angle with $\vec e$ for all $\vec x$.

\section{Conclusion}
This paper studies the problem of neural network verification, for which the existing cut based techniques to reduce relaxation gap fail to provide meaningful improvement in accuracy. We leverage the fact that loose candidate optimum points in an SDP relaxation of a QCQP reformulation of the problem can be invalidated by simple linear constraints. Using this property, a SDP-based method is developed, which reduces the relaxation gap to zero as the number of iterations increases. This algorithm is proven to be theoretically and empirically effective.

By actively constructing constraints using CGLP, we obtain provably valid cuts. More importantly, a negative optimal value in CGLP also guarantees a nonzero reduction in the relaxation gap. It is verified that the relaxation gap for the existing methods is large when tested on large-scale networks, while the proposed algorithm offers a satisfactory performance. 

\bibliography{references}

\end{document}